\documentclass[sigconf]{acmart}
\AtBeginDocument{%
  \providecommand\BibTeX{{%
    \normalfont B\kern-0.5em{\scshape i\kern-0.25em b}\kern-0.8em\TeX}}}

\setcopyright{iw3c2w3}
\copyrightyear{2020}
\acmYear{2020}
\acmDOI{10.1145/3366423.3380213}

\acmConference[WWW '20]{Proceedings of The Web Conference 2020}{April 20--24, 2020}{Taipei, Taiwan}
\acmBooktitle{Proceedings of The Web Conference 2020 (WWW '20), April 20--24, 2020, Taipei, Taiwan} 
\acmPrice{}
\acmISBN{978-1-4503-7023-3/20/04}

\usepackage{tikz}
\usepackage{pgfplots}
\usetikzlibrary{backgrounds}
\usetikzlibrary{bayesnet}
\usetikzlibrary{arrows}
\usepackage{subcaption}
\usepackage{algorithm}
\usepackage{algpseudocode}

\newtheorem{theorem}{Theorem}

\begin{document}

\title{Modeling Heterogeneous Statistical Patterns in High-dimensional Data by Adversarial Distributions: An Unsupervised Generative Framework}

\author{Han Zhang$^1$, Wenhao Zheng$^3$, Charley Chen$^1$, Kevin Gao$^1$, Yao Hu$^3$, Ling Huang$^2$ and Wei Xu$^1$}
\affiliation{%
  \institution{$^1$Tsinghua University, Beijing, China \\ $^2$AHI Fintech, China,
       $^3$Youku Cognitive and Intelligent Lab, Alibaba Group}
  }
\email{ {han-zhan17, ctj2015, kevingao96}@mails.tsinghua.edu.cn,  {zwh149850, yaohu}@alibaba-inc.com,huang.ling@gmail.com,weixu@tsinghua.edu.cn}

\renewcommand{\shortauthors}{Zhang and Zheng, et al.}
\renewcommand{\shorttitle}{Modeling Heterogeneous Statistical Patterns In High-dimensional Data By Adversarial Distributions}

\begin{abstract}
Since the label collecting is prohibitive and time-consuming, unsupervised methods are preferred in applications such as fraud detection.
Meanwhile, such applications usually require modeling the intrinsic clusters in high-dimensional data, which usually displays heterogeneous statistical patterns as the patterns of different clusters may appear in different dimensions.
Existing methods propose to model the data clusters on selected dimensions, yet globally omitting any dimension may damage the pattern of certain clusters.
To address the above issues, we propose a novel unsupervised generative framework called FIRD, which utilizes adversarial distributions to fit and disentangle the heterogeneous statistical patterns.
When applying to discrete spaces, FIRD effectively distinguishes the synchronized fraudsters from normal users.
Besides, FIRD also provides superior performance on anomaly detection datasets compared with SOTA anomaly detection methods (over 5\% average AUC improvement).
The significant experiment results on various datasets verify that the proposed method can better model the heterogeneous statistical patterns in high-dimensional data and benefit downstream applications.

\end{abstract}




\maketitle
\section{Introduction}
Human-annotated labels are widely adopted to supervise machine learning.
Yet, large-scale labeled datasets are usually prohibitive and time-consuming to obtain~\cite{deng2009imagenet,socher2013recursive,lin2014microsoft,rajpurkar2016squad,wang2018glue}, so the collected labels may be obsolete for training.
For example, online platforms label users as fraudsters after detecting certain malicious behaviors.
However, once the fraudsters are exposed,  they tend to change their strategies of committing fraud so that the platforms cannot utilize these labels to identify fraudsters in the future.
As such, unsupervised methods are usually preferred in such applications to perform real-time pattern recognition.

Applications such as fraud detection usually require modeling the intrinsic clusters in high-dimensional data. 
Such data usually displays \textit{heterogeneous statistical patterns} as the patterns of different clusters may appear in different dimensions.
As an example, consider the registration log data in an online platform which contains dozens of dimensions in 4 types: user information (name, ID card number, gender, age, phone number), device information (MAC address, OS type, manufacturer), network information (IP address, channel, browser) and behavior information (timestamp, time elapsed during registration).
The fraudsters usually display abnormal \textit{synchronized behaviors} in specific dimensions due to similar control scripts or resource sharing, while the normal users are randomly distributed~\cite{palshikar2002hidden,raj2011analysis}.
For instance, one fraud group may utilize the Android simulator to implement bulk registration, which in turn leads to similar device and behavior information.
Another fraud group may hire people to register the accounts manually.
As they possess only a few available phone numbers and network proxies, their network and user identity signatures will be similar.

The heterogeneous statistical patterns are challenging to identify, as such fraudsters will try their best to get disguised as normal users.
In practice the fraudsters are highly indistinguishable from the normal users concerning all the dimensions due to two reasons: 1) many more dimensions are recorded in real-world scenarios (e.g., more detailed personal information and historical behaviors from other platforms), and 2) normal users may occasionally share feature values.
From a global perspective, traditional similarity-based clustering methods may result in too many false positives as they are easily affected by noisy normal users.
Many previous works propose to omit uninformative features and cluster the dataset on the remaining features~\cite{raftery2006variable,jovic2015review,alelyani2013feature}.
As the local synchronization of different fraud groups is usually distinct (e.g., device/time information and network/user information), it is often hard to tell which dimension to omit.
In fact, the fraud patterns evolve as fraudsters learn which dimensions are ignored and thus find new ways to avoid detection.
Therefore, to effectively recognize the heterogeneous statistical patterns, methods should model both the global distribution of all data clusters and the local patterns for each cluster.

Modeling such heterogeneous statistical patterns is challenging for two reasons.
First, globally modeling the data clusters requires evaluating the joint distribution of all dimensions, which suffers from the curse of dimensionality.
For example, the number of possible combinations of (gender, age, phone number, \dots ) is the product of each dimensions' possible value counts.
Therefore, the number of parameters required to describe such joint distribution grows exponentially fast with the dimensionality.
Second, recognizing the local cluster patterns requires modeling the mixture distribution of the different fraud groups and random normal users, which is usually distinct among the dimensions as the dimensions' sample spaces (i.e., the set of all possible values) are different.
For example, the gender and age of a user take values respectively from a binary-value set \{\textit{male}, \textit{female}\} and a ten-value set \{0\textasciitilde 10, 10\textasciitilde 20, \dots , 90\textasciitilde 100\}, so we need 1 and 9 parameters to describe the distributions on these two dimensions correspondingly.
As a result, conventional mixture methods such as Gaussian mixture model (GMM) are inapplicable to such data, since they assume that all dimensions of the data share the same sample space, e.g., $\mathbb{R}$ for GMM.

To address the above issues, we propose a novel generative framework called FIRD to model the heterogeneous statistical patterns in unlabelled datasets.
\textbf{FIRD} effectively models both the global and local patterns based on the \textbf{F}eature \textbf{I}ndependence assumption and the adve\textbf{R}sarial \textbf{D}istributions.
FIRD assumes the features are conditionally independent within the same group to model the dimensions with different sample spaces and relieve the curse of dimensionality.
It then utilizes the {\it adversarial distributions} (a group of distributions competing in generating observations) to fit and disentangle the complex data distributions in each dimension, which in turn brings the model stronger interpretability.
Specifically, FIRD identifies the synchronization of the fraud groups as well as the randomly distributed normal users in discrete space by using a pair of adversarial multinomial distributions in each dimension.
It is also worth emphasizing that FIRD is not limited to fraud detection tasks.
We demonstrate in our experiments that modeling the heterogeneous statistical patterns are also beneficial to anomaly detection tasks (over 5\% average AUC improvement over SOTA methods).
We expect FIRD to be effective in other applications that model the patters other than the synchronization and randomness by adopting appropriate adversarial distributions.
The learned probabilistic representations also support probabilistic reasoning based on prior information.
The major contributions of this paper are distinguished as follows:
\begin{itemize}
  \item We present a novel generative framework FIRD, which adopts adversarial distributions to capture heterogeneous statistical patterns in unlabeled datasets.
  \item In discrete spaces, FIRD provides interpretable fraud detection and anomaly detection results on various datasets.
  \item The effectiveness of FIRD indicates modeling the heterogeneous statistical patterns benefits various downstream tasks.
\end{itemize}

The remaining of the paper is arranged as follows.
We review the related works in Section \ref{sec:related_work}.
We introduce the proposed framework FIRD in Section~\ref{sec:method}.
We demonstrate the experimental results in Section~\ref{sec:experiments}.
We conclude the paper in Section \ref{sec:conclusion}.

\section{Background}
\label{sec:related_work}
In this section, we review the approaches related to the method and applications discussed in this paper.
Then we demonstrate the identifiability issue arisen in the existing generative models.

\begin{table*}[!tb]
\begin{tabular}{cl|cl|cl}
\toprule
Symbol & Meaning & Symbol & Meaning & Symbol & Meaning\\
\midrule
$n$ & index of data & $N$ & total number of data samples & $\mathbf{x}_n$ & observed data\\
$g$ & index of clusters & $G$ & total number of clusters & $d_n$ & hidden cluster indicator \\
$m$ & index of features & $M$ & total number of features & $f_{nm}$ & adversarial distribution indicator\\
$i$ & index of feature values & $D_m$ & $m$-th feature dimension & $\mathcal{F}_m$ & $m$-th dimension\\
$p(\cdot)$ & generative model & $p_k(\cdot)$ & $k$-th generative component & $q(\cdot)$ & variational distributions \\
$\pi_g$ & mixture weight of $g$-th cluster& $\mu_{gm}$ & adversarial distribution weight & $\boldsymbol{\alpha}_{gm}$ & synchronized multinomial parameter\\
$\boldsymbol{\beta}_{gm}$ & random multinomial parameter & $\widetilde{\phi}_{ng}$ & responsibility of $g$-th cluster & $\widetilde{\gamma}_{ngm}$ & responsibility of $m$-th feature\\
$\gamma_{ngm}$ & auxiliary variable for E step & $\bar{\gamma}_{ngm}$ & auxiliary variable for E step & $\lambda$ & hyper parameter for  prior regularization\\
\bottomrule
\end{tabular}    
\caption{Symbols in this paper. The bold uppercase represents a matrix, the bold lowercase represents a vector and the regular type represents a scaler. Parameters with hats (e.g., $\hat{\pi}_g$) are the parameters in the last EM iteration.}
\label{tab:symbols}
\end{table*}

\subsection{Related Approaches}
Recognizing the heterogeneous statistical patterns in unlabeled datasets requires modeling both the global distribution of the data clusters and local patterns of each data cluster.
When the local patterns of all data clusters are distributed in the same subset of the features, we may omit other non-informative features and perform clustering on the remaining features~\cite{raftery2006variable,jovic2015review,alelyani2013feature}.
The feature selection algorithms in clustering are basically divided into three categories: the filter models, wrapper models, and hybrid models~\cite{alelyani2013feature}.
The filter models use certain criterion to evaluate the quality of the features and then cluster the data w.r.t. high-quality features~\cite{zhao2007spectral}.
The wrapper models enumerate all feature combinations and utilize a specific clustering method to evaluate each combination~\cite{kim2002evolutionary}.
Unlike the other two models, the hybrid models simultaneously select useful features and cluster the data points.
For example, feature saliency models attempt to fit each feature with either a Gaussian mixture model (GMM) or a global Gaussian~\cite{law2003feature,tadesse2005bayesian,constantinopoulos2006bayesian,silvestre2015feature,white2016bayesian}.
In this way, features fitted by the GMM are effective for clustering, while other features fitted by the global Gaussian are discarded.
However, such methods are unable to model the heterogeneous statistical patterns among the dimensions for three reasons.
First, since the local patterns of the clusters may involve different dimensions, no dimension should be omitted globally.
Second, distributions like multivariate Gaussian assume the dimensions share the same sample space, such as $\mathbb{R}$, which is usually not true in practice.
Finally, such model-based methods suffer from an identifiability problem when applying to discrete spaces~\cite{silvestre2015feature,white2016bayesian}, as they fail to disentangle the clusters with a mixture of multinomials (Section~\ref{sec:identifiability}).

The synchronized fraudsters described above can be seen as dense blocks (i.e., groups of frequently shared feature values) in the dataset.
Related works then detect fraud groups by searching for the dense blocks from the dataset filled with randomly distributed normal users.
\cite{jiang2015crossspot} measures a block's density by its likelihood under a Poisson assumption and greedily searches for blocks to optimize it.
\cite{shin2016mzoom} proposes a greedy optimization framework to optimize a given density measure. 
D-Cube \cite{shin2017dcube}, the successor to \cite{shin2016mzoom}, has accelerated computation and is thus applicable to much larger datasets.
However, these greedy search methods focus only on the fraud patterns, so the changes in the patterns of the normal users can easily affect these models.
Specifically, the precision of these methods declines significantly with the growing number of normal users, as the randomly distributed normal users are noise that can significantly interfere with the searching process (Section \ref{sec:fraud_comparison}).

Modeling the heterogeneous statistical patterns can also benefit anomaly detection methods.
Anomaly detection methods assume that the majority of a dataset is closely distributed on several specific manifolds~\cite{aggarwal2013outlier}.
Data points distant from these manifolds are thus identified as anomalies.
With this manifold assumption, one can spot the anomalies by linear methods~\cite{shyu2003novel,dufrenois2016one}, proximity-based methods~\cite{he2003discovering,goldstein2012histogram}, and outlier ensembles ~\cite{lazarevic2005feature,liu2008isolation,zimek2013subsampling}.
In high-dimensional datasets, however, only when considering specific dimensions, the anomalies are distant from the manifolds.
To select informative features in the outlier detection problem, many previous works propose the ensembles of the base detectors~\cite{lazarevic2005feature,liu2008isolation,zimek2013subsampling}.
The basic thinking is resampling the data points and the features so that the model takes merit from different feature combinations and reduces the fraction of anomalies.
However, such resampling searches through all feature combinations, which is exponential to the number of dimensions.
Besides, when the resampling fails to generate the right feature combination, the detection performance declines significantly.

\subsection{Identifiability Problem in Existing Generative Methods}
\label{sec:identifiability}
Two model-based feature selection methods are proposed for clustering discrete data~\cite{silvestre2015feature,white2016bayesian}.
They extend the seminal works~\cite{law2003feature,tadesse2005bayesian} from continuous to the discrete case by replacing Gaussian distributions with multinomial distributions.
Here we show the identifiability issue of such models in discrete spaces.

To avoid confusion, we use the same denotation as the~\cite{silvestre2015feature}.
The likelihood is
\begin{equation}
    \mathcal L(\boldsymbol\Theta; \mathbf{y}_i) = \prod_{i=1}^N\sum_{k=1}^K\alpha_k\prod_{l=1}^L[\rho_lp(y_{il}|\theta_{lk}) + (1-\rho_l)q(y_{il}|\theta_l)],
\end{equation}
where $\mathbf{y}$ is the observation. 
Parameters $\boldsymbol\alpha$ and $\boldsymbol{\rho}$ are the mixture weights and indicator parameters of feature selection, respectively.
They use $\theta_{lk}$ to represent the group-specific parameters and $\theta_{l}$ to represent the global parameters.
Here the distributions $p$ and $q$ are both multinomials.
We can rewrite the probability as
\begin{equation}
    p(y_{il}=Y_{lm}) = \sum_{k=1}^K\alpha_k[\rho_l\theta_{lkm} + (1-\rho_l)\theta_{lm}],
\end{equation}
where $\theta_{lkm}$ and $\theta_{lm}$ are the probability mass of $Y_{lm}$.
Then we can write the likelihood in a different way as
\begin{equation}
    \mathcal L(\boldsymbol\Theta; \mathbf{y}_i) = \prod_{i=1}^N\prod_{l=1}^L\prod_{m=1}^{M_l}\left\{\sum_{k=1}^K\alpha_k[\rho_l\theta_{lkm} + (1-\rho_l)\theta_{lm}]\right\}^{\mathbf{1}_{\{y_{il}=Y_{lm}\}}},
\end{equation}
where the function $\mathbf{1}_{\{\cdot\}}$ is the indicator function.
If we have already learned the optimal $\boldsymbol{\rho}^\star$, and denoting
\begin{displaymath}
    \tilde{p}_{lm} = \sum_{k=1}^K\alpha_k[\rho^\star_l\theta_{lkm} + (1-\rho^\star_l)\theta_{lm}],
\end{displaymath}
we see that optimizing such mixture of multinomials is equivalent to optimizing a single multinomial with parameter $\tilde{\mathbf{p}}_l = [\tilde{p}_{l1} \dots \tilde{p}_{lm} \dots \tilde{p}_{lM}]$ for each feature.
If we introduce information criteria to determine $G$ automatically as they did in the original method, $G$ will always be 1.
The existing methods introduce informative priors merely to the mixture weights and feature selection indicator parameters, which does not solve the identifiability issue.

\section{Our Method}
\label{sec:method}

In this section, we first present a novel unsupervised generative learning framework called FIRD and describe how it captures the heterogeneous statistical patterns using the adversarial distributions.
Then we apply the framework to discrete space to model two specific patterns, namely synchronization, and randomness, for applications such as fraud detection and anomaly detection.
For better readability, Table~\ref{tab:symbols} summarize the meanings of the symbols in this paper.

\subsection{FIRD: A Novel Learning Framework}
\label{sec:framework}
In this section, we present a novel unsupervised generative learning framework called FIRD.
We first introduce the adversarial distributions and discuss the identifiability issue caused by over-flexible variational distributions.
Then we propose to solve this identifiability problem using prior knowledge and provide theoretical guarantees.

For each observation in $\mathcal{D}=\{\boldsymbol{x}_n\}_{n=1}^N$, we want to learn its associated latent semantic variable $d_n$.
To characterize the different sample space and relieve the curse of dimensionality, we assume the data features are independent given $d_n$, i.e. $p(\mathbf{x}_n|d_n)=\prod_{m=1}^Mp(x_{nm}|d_n)$.
We assume the observations display up to $K$ distinct statistical patterns in each dimension.
To characterize and balance these patterns, we introduce the {\it adversarial distributions} $\left\{p_k(x_{nm}|d_n)\right\}_{k=1}^K$, such that given $d_n$, the distribution of observation $x_{nm}$ is a mixture of the adversarial distributions, i.e. $p(x_{nm}|d_n)=\sum_{k=1}^K\mu_{mk}p_k(x_{nm}|d_n)$.
Here $\mu_{mk}$ determines the responsibility of each adversarial distribution component $p_k(x_{nm}|d_n)$ for generating $x_{nm}$.
As $\boldsymbol{\mu}_m$ lies in a probability simplex, increasing the responsibility $\mu_{mk}$ of $p_k(x_{nm}|d_n)$ will reduce that of other adversarial distribution components.
In this way, the adversarial distributions compete in generating the observations by fitting the intricate data patterns in each dimension.
Under such assumption, the data generation process can be described by the following two steps:
\begin{enumerate}
  \item Generate $d_n$ from the semantic distribution $p(d_n)$.
  \item For each feature $m$, choose an adversarial component $k$ w.r.t. $\mu_{km}$, and generate $x_{nm}$ from $p_k(x_{nm}|d_n)$.
\end{enumerate}
Since $p(d_n)$ and $p_k(x_{nm}|d_n)$ are unknown, we cannot directly evaluate the likelihood of such generative models.
We then turn to optimizing the likelihood lower bound using variational methods.
Denoting the parameters of the learner as $\pi$ and $\theta$, we have the evidence lower bound of the log-likelihood (ELBO):
\begin{equation}
\label{eq:elbo}
\begin{aligned}
&\log\mathcal{L}(\theta, \mathbf{d};\mathbf{x})=\sum_{n=1}^N\log \sum_{d_n}p(d_n|\pi)\prod_{m=1}^M\sum_{k=1}^K\mu_{mk} p_k(x_{nm}|d_n,\theta)\\
&\geq \sum_{n=1}^N\sum_{d_n}q(d_n)\left\{\log\frac{p(d_n, \pi)}{q(d_n)}+\sum_{m=1}^M\sum_{k=1}^K\tilde{\mu}_{nmk}\log \frac{\mu_{mk} p_k(x_{nm}|d_n, \theta)}{\tilde{\mu}_{nmk}}\right\},
\end{aligned}
\end{equation}
where we introduce variational distribution $q(d_n)$ to approximate the posterior $p(d_n|\mathbf{x}_n)$ and apply Jensen's inequality with auxiliary variables $\tilde{\mu}_{nmk}$ ($\sum_{k}\tilde{\mu}_{nmk}=1, \tilde{\mu}_{nmk}\geq 0$).

As the ELBO is concave, we can use the EM algorithm to iteratively approximate the latent semantic distribution $p(x_n|d_n)$ and estimate the parameters.
However, if the adversarial distributions are too flexible, we may fail to disentangle the complex data distribution as expected.
An extreme case is that if $p_k(x_{nm}|d_n, \theta)$ is able to fit the data patterns perfectly, i.e. $\inf_{\theta}D_{KL}\left(p_k(x_{nm}|d_n, \theta)\middle\Vert p^\star(x_{nm}|d_n)\right)=0$, it can dominate the adversarial distributions with $\mu_{mk}=1$ and $\mu_{mj}=0$ for all $j\neq k$, where we denote the ground truth distributions as $p^\star(\cdot)$.
To address this problem, we introduce prior information into these adversarial distributions such that each adversarial component can fit one specific pattern much better than other patterns.
We have the following theorem to guarantee the optimality and uniqueness of the variational approximation.

\begin{theorem}
\label{thm:unique_optimizer}
Suppose that with proper priors, $\forall k \in \{1,\dots,K\}$ there exists a unique and distinct index $k'\in \{1,\dots,K\}$ and some constant $C>0$ such that
\begin{displaymath}
\begin{aligned}
    \inf_{\theta}D_{KL}\left(p^\star_{k'}(x_{nm}|d_n)\middle\Vert p_j(x_{nm}|d_n, \theta)\right) &> C, \\
    \inf_{\theta}D_{KL}\left(p^\star_{k'}(x_{nm}|d_n)\middle\Vert p_k(x_{nm}|d_n, \theta)\right) &= 0
\end{aligned}
\end{displaymath}
for all $j\neq k$, where $D_{KL}(\cdot\Vert\cdot)$ is the KL divergence.
Then we have the unique optimal solution:
\begin{displaymath}
\begin{aligned}
q(d_n)=p^\star(x_{nm}&|d_n),~p_k(x_{nm}|d_n, \theta) = p_k^\star(x_{nm}|d_n)\\
\tilde{\mu}_{nmk}&=\frac{\mu_{mk}p_k^\star(x_{nm}|d_n)}{\sum_{k=1}^K\mu_{mk}p_k^\star(x_{nm}|d_n)}
\end{aligned}
\end{displaymath}
for the ELBO in Eq.~\eqref{eq:elbo}.
\end{theorem}
\begin{proof}
Please see Appendix~\ref{sec:theorem_proof}.
\end{proof}

To address the specific modeling problem in applications such as fraud detection, we apply FIRD to the discrete spaces.

\subsection{Discrete Space Application}
\label{sec:fird}
In this section, we apply FIRD to discrete spaces for applications such as fraud detection.
We first illustrate the model and optimization techniques.
After that, we describe how FIRD reduces the noise in the dataset by filtering outliers.
Then we introduce how FIRD makes inference based on the prior knowledge and the learned probability representations.
Finally, we discuss the complexity, initialization, and hyperparameters.

\subsubsection{Model}
The FIRD model in discrete space is a finite mixture model with adversarial multinomial distribution pairs.
To better illustrate the generation process, we plot the plate representation of the FIRD in discrete spaces in \figurename~\ref{fig:GraphicModel}.

\begin{figure}[!tb]
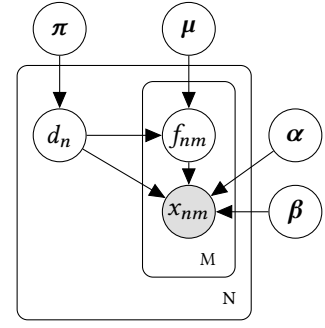

  \centering
  \tikz{ %
    \node[latent] (d) {$d_n$} ; %
    \node[latent, right=of d] (f) {$f_{nm}$} ; %
    \node[obs, below=.3cm of f] (x) {$x_{nm}$} ; %
    \node[latent, above=.7cm of d] (pi) {$\boldsymbol\pi$};
    \node[latent, above=.7cm of f] (mu) {$\boldsymbol\mu$};
    \node[latent, right=.7of f] (alpha) {$\boldsymbol\alpha$};
    \node[latent, right=.7of x] (beta) {$\boldsymbol\beta$};
    \plate[inner sep=0.25cm, yshift=.1cm] {plate1} {(f) (x)} {M}; %
    \plate[inner sep=0.2cm] {plate2} {(d) (plate1)} {N}; %
    \edge {d} {f} ; %
    \edge {d, f, alpha, beta} {x} ; %
    \edge {pi} {d};
    \edge {mu} {f};
  }
 \caption{The plate representation of FIRD in discrete space.
 The plates with subscript $M$ and $N$ indicate respectively the $M$ independent features and $N$ i.i.d. data points.
 The parameter $\boldsymbol\pi$ is the mixture weight.
 $\boldsymbol{\mu}$ balances the adversarial distribution pairs.
 The synchronization and randomness are captured by the adversarial distribution pairs, whose parameters are $\boldsymbol{\alpha}$ and $\boldsymbol{\beta}$, respectively.}
 \label{fig:GraphicModel}
\end{figure}

In discrete spaces, the observations $\mathcal{D}=\{\mathbf{x}_n\}_{n=1}^N$ are $M$-dimensional vectors, where for each feature $\mathcal{F}_m$, $1\leq m \leq M$, $x_{nm}$ takes its value from the set $\{X_{mi}\}_{i=1}^{D_m}$.
The latent semantic variable $d_n$ takes discrete values from $\{1,\dots,G\}$ governed by $p(d_n) = Mult(\boldsymbol{\pi})$, where we use $Mult(\cdot)$ for abbreviation of a multinomial distribution.
We assume that given $d_n$, each feature $\mathcal{F}_m$ independently generates the feature value from an adversarial multinomial distribution pair that respectively captures synchronization and randomness.
Suppose the synchronization-aware components of the adversarial distribution pairs are multinomials controlled by $\boldsymbol\alpha$, and the randomness-aware components are multinomials controlled by $\boldsymbol\beta$. 
We then have the following generation process for an observation $\mathbf{x}_n$:
\begin{enumerate}
\item Choose the semantic variable $d_n\sim$ $Mult(\boldsymbol\pi)$.
\item For each feature $\mathcal{F}_m$:
\begin{enumerate}
\item Choose the indicator variable $f_{nm}\sim$ $Bernoulli(\mu_{d_nm})$.
\item If $f_{nm}=1$, choose the observed value $x_{nm}\sim$ $Mult(\boldsymbol\alpha_{d_nm})$;
\item If $f_{nm}=0$, choose the observed value $x_{nm}\sim$ $Mult(\boldsymbol\beta_{d_nm})$.
\end{enumerate}
\end{enumerate}

As described in section~\ref{sec:framework}, we need to introduce priors into the adversarial distributions so that each component of the adversarial distributions fits one specific pattern.
In discrete spaces, to generate synchronized observations, the probability mass of the distribution $Mult(\boldsymbol\alpha)$ should distribute on a small subset of the feature values.
In other words, most entries in $\boldsymbol{\alpha}$ should be zero so that sampling from $Mult(\boldsymbol\alpha)$ leads to several possible outputs.
Similarly, to model the randomness $Mult(\boldsymbol\beta)$ should approximate the uniform distribution, so that sampling from $Mult(\boldsymbol\beta)$ produces a wide variety of values.
We can achieve such randomness- and synchronization-awareness by introducing Dirichlet-like priors to promote sparse $\boldsymbol\alpha_{d_nm}$ and smooth $\boldsymbol\beta_{d_nm}$.
Moreover, according to the traditional information criterions, we can promote sparse $\boldsymbol{\pi}$ to automatically choose a proper $G$ for the dimension of $d_n$ (see for example~\cite{barron1998minimum}).
\begin{equation}
\begin{aligned}
\label{eq:objective}
    &~~~~\log\mathcal{L}(\boldsymbol\pi,\boldsymbol\mu,\boldsymbol\alpha,\boldsymbol\beta; \mathcal D, \boldsymbol\lambda^{(1)}, \boldsymbol\lambda^{(2)}) \\
    &= \sum_{n=1}^N\log\bigg\{\sum_{d_n, \mathbf{f}_n}
    p(\mathbf{x}_{n}|\mathbf{f}_{n}, d_n, \boldsymbol\alpha, \boldsymbol\beta)p(\mathbf{f}_{n}|d_n, \boldsymbol\mu)p(d_n|\boldsymbol\pi)
    \bigg\}\\
    &~~~~~-\sum_{g=1}^G \lambda_{g}^{(1)}\log\pi_g-\sum_{g=1}^G\sum_{m=1}^M\sum_{i=1}^{D_m}\lambda_{gmi}^{(2)}\left(\log\alpha_{gmi}-\log\beta_{gmi}\right).
\end{aligned}
\end{equation}
We introduce the hyper parameter $\boldsymbol{\lambda}$ to control the regularization.
Note that the regularizers for $\boldsymbol{\pi}$ and $\boldsymbol{\alpha}$ (last two terms in Eq. \eqref{eq:objective}) are negative to promote sparsity, and that for $\boldsymbol{\beta}$ is positive to promote randomness.
We choose the same weight $\lambda_{gmi}^{(2)}$ for each adversarial distribution pair to ensure they are equally capable of modeling the synchronization and randomness.
We can determine whether $Mult(\boldsymbol{\alpha})$ or $Mult(\boldsymbol{\beta})$ prevails by introducing priors to parameter $\boldsymbol{\mu}$, which makes FIRD fully Bayesian.
However, in this paper, we just adopt a fair competition between the adversarial distributions.

\subsubsection{Optimization}
We then apply the EM algorithm for optimization.
In the E-step, we calculate the posterior distribution of the latent variables given the observations, i.e. $p(d_n, \mathbf{f}_{n}|\mathbf{x}_{n}, \hat{\boldsymbol\theta})$.
To represent this posterior distribution we estimate
\begin{equation}
\label{eq:latent_info}
\begin{aligned}
\widetilde{\phi}_{ng}&=p(d_n=g|\mathbf{x})=\frac{\hat{\pi}_{g}\prod_{m=1}^M\left\{\gamma_{ngm}+\bar{\gamma}_{ngm}\right\}}{\sum_{g'=1}^G\hat{\pi}_{g'}\prod_{m=1}^M\left\{\gamma_{ng'm}+\bar{\gamma}_{ng'm}\right\}},\\
\widetilde{\gamma}_{ngm} &=p(f_{nm}=1|\mathbf{x}, d_n=g) =\frac{\gamma_{ngm}}{\gamma_{ngm}+\bar{\gamma}_{ngm}},
\end{aligned}
\end{equation}
where we denote the parameter estimations from the last iteration with hats and define
\begin{displaymath}
\begin{aligned}
\gamma_{ngm} &= \hat{\mu}_{gm}\prod_{i=1}^{D_m}\hat{\alpha}_{gmi}^{x_{nmi}}, 
\bar{\gamma}_{ngm} = (1-\hat{\mu}_{gm})\prod_{i=1}^{D_m}\hat{\beta}_{gmi}^{x_{nmi}}. \\
\end{aligned}
\end{displaymath}

In the M-step, we find parameters that optimize the expected likelihood.
Since the objective is concave w.r.t. $\boldsymbol{\mu}$ and $\boldsymbol{\beta}$, setting the derivative to zero gives us
\begin{equation}
\label{eq:update_one}
\begin{aligned}
\mu_{gm} &= \frac{\sum_{n=1}^N\widetilde{\gamma}_{ngm}\cdot\widetilde{\phi}_{ng}}{\sum_{n=1}^N\widetilde{\phi}_{ng}}, \\
\beta_{gmi}&= \frac{\lambda_{gmi}^{(2)} + \sum_{n=1}^Nx_{nmi}(1-\widetilde{\gamma}_{ngm})\widetilde{\phi}_{ng}}{D_{m}\lambda_{gmi}^{(2)} + \sum_{n=1}^N(1-\widetilde{\gamma}_{ngm})\widetilde{\phi}_{ng}}.\\
\end{aligned}
\end{equation}
The optimization of $\boldsymbol{\pi}$ and $\boldsymbol{\alpha}$ does not have a closed-form solution, as the objective is no longer concave due to the regularizers.
However, by introducing some small noise to the Dirichlet-like priors, we can calculate the maximum using a numerical method~\cite{larsson2011concave}.
We iteratively update $\boldsymbol{\pi}$ and $\boldsymbol{\alpha}$
\begin{equation}
\label{eq:update_two}
\begin{aligned}
\pi_g &= \frac{\sum_{n=1}^N\widetilde{\phi}_{ng} + \lambda_g^{(1)}\pi_g}{N + \lambda_g^{(1)}/\pi_g},\\
\alpha_{gmi} &= \frac{\sum_{n=1}^Nx_{nmi}\cdot\widetilde{\gamma}_{ngm}\cdot\widetilde{\phi}_{ng} + \lambda_{gmi}^{(2)}\alpha_{gmi}}{\sum_{n=1}^N\widetilde{\gamma}_{ngm}\cdot\widetilde{\phi}_{ng} + \lambda_{gmi}^{(2)}/\alpha_{gmi}},
\end{aligned}
\end{equation}
until convergence.
We present the detailed derivation in the supplementary material\footnote{Supplementary material is available at https://github.com/fingertap/fird.cython/blob/\\master/suppl\_material.pdf.}, and we show the pseudo-code of FIRD in Algorithm~\ref{alg:fird}.

\begin{algorithm}
\caption{FIRD in Discrete Space}
\label{alg:fird}
\begin{algorithmic}[1]
\Require Observations $\mathcal{D}=\{\mathbf{x}_n\}_{n=1}^N$, semantic dimension $G$, regularization weight $\boldsymbol{\lambda}$, precision $\epsilon$.
\Ensure Semantic representations $\widetilde{\boldsymbol\phi}$ and $\widetilde{\boldsymbol{\gamma}}$, semantic distirbution $\boldsymbol{\pi}$, balance parameters $\boldsymbol{\mu}$ of the adversarial distributions, adversarial distribution parameters $\boldsymbol{\alpha}$ and $\boldsymbol{\beta}$.
\For {$g=1$ to $G$, $m=1$ to $M$} \Comment{Initialize parameters.}
    \State $\pi_g\gets 1/G$, $\mu_{gm}\leftarrow 0.5$.
    \State Randomly initialize and normalize $\boldsymbol{\alpha}_{gm}$ and $\boldsymbol{\beta}_{gm}$.
\EndFor
\State $\mathcal{L}\gets -\infty$\Comment{Initialize likelihood.}
\Repeat
\State Evaluate $\mathcal{L}^{new}$ according to Eq.~\eqref{eq:objective}.
\State Calculate $\widetilde{\boldsymbol\phi}$, $\widetilde{\boldsymbol{\gamma}}$ according to Eq.~\eqref{eq:latent_info}.\Comment{E-step.}
\State Update $\boldsymbol{\mu}$ and $\boldsymbol{\beta}$ according to Eq.~\eqref{eq:update_one}.\Comment{M-step.}
\State Update $\boldsymbol{\pi}$ and $\boldsymbol{\alpha}$ according to Eq.~\eqref{eq:update_two}.\Comment{M-step.}
\Until {$\mathcal{L}^{new}-\mathcal{L} < \epsilon$}
\end{algorithmic}
\end{algorithm}

\subsubsection{Noise Reduction}
\label{sec:refine}
Since the real-world datasets always contain a considerable amount of noisy data points, here we describe how FIRD deals with such noisy datasets by filtering the outliers.

Observations that do not belong to any possible values of the latent semantic variable are recognized as outliers, which are noise to be removed from the dataset.
Specifically, the likelihood of the observation given $d_n=g$ is
\begin{displaymath}
\begin{aligned}
p(\mathbf{x}_{n}|d_n=g)=
\prod_{m=1}^M\left\{\gamma_{ngm} + \bar{\gamma}_{ngm}\right\}.
\end{aligned}
\end{displaymath}
Given $d_n=g$, the information gain after observing $\mathbf{x}_n$ is
\begin{displaymath}
I(\mathbf{x}_n|d_n=g) = - \log p(\mathbf{x}_{n}|d_n=g).
\end{displaymath}
We expect this information to be large for outliers, so we compute a threshold on $I$.
A natural choice is its expectation, i.e. the entropy of the distribution $p(\mathbf{x}_{n}|d_n=g)$:
\begin{equation}
\label{eq:threshold}
\begin{aligned}
&H[p(\mathbf{x}_{n}|d_n=g)] = \sum_{m=1}^M H[p(x_{nm}|d_n=g)] \\
= &-\sum_{m=1}^M\sum_{i=1}^{D_m}h\left(\mu_{gm}\alpha_{gmi} + (1-\mu_{gm})\beta_{gmi}\right),
\end{aligned}
\end{equation}
where we defined $h(y) = y\log y$.
Then we can filter out the outliers that satisfy $I(\mathbf{x}_n|d_n=g) > (1+\epsilon) \cdot H[p(\mathbf{x}_n|d_n=g)]$ for all components, where $\epsilon$ is the tolerance.

\subsubsection{Inference}
\label{sec:prior}
We can incorporate our prior knowledge of $\mathbf{d}$ to make inference based on the learned probabilistic representations.
Using classification as an example, we can infer the label of each data point given the labels of $\mathbf{d}$:
\begin{equation}
\label{eq:label}
\ell_n\triangleq \mathbb{E}_{d_n}\left[\ell|\mathbf x_n\right] = \sum_{g=1}^Gp(\ell|d_n=g)p(d_n=g|\mathbf{x}_n),
\end{equation}
where $\ell$ is application-dependent target variable.
Since $G \ll N$, domain experts can efficiently analyze the learned latent representations $\mathbf{d}$ and design appropriate decision distributions $p(\ell|d_n=g)$.

In fraud detection, we have the prior knowledge that fraudsters usually synchronize with each other due to potential resource sharing~\cite{palshikar2002hidden,raj2011analysis}.
This allows us to determine $p(\ell|d_n=g)$ automatically by calculating the ``difficulty of generation'' under a random model $\mathcal{M}_{\text{random}}$.
If the observations are too difficult to generate under the random model, they are very likely to be fraudsters.
Suppose the values $X_{mi}$ have probability $\frac{1}{D_m}$ for each $\mathcal{F}_m$ in the random model $\mathcal{M}_{\text{random}}$.
Then the probability of generating $g$ is
\begin{displaymath}
    p(d_n=g|\mathcal{M}_{\text{random}}) = \prod_{n=1}^N\pi_g\prod_{m=1}^M\prod_{i=1}^{D_m}\binom{D_m}{N_{gmi}}D_m^{-N_{gmi}},
\end{displaymath}
where we have defined
\begin{displaymath}
    N_{gmi} = \sum_{n=1}^N\mathbf{1}(x_{nm}=X_{mi})\cdot p(d_n=g|\mathbf{x}_n).
\end{displaymath}
Note that $\binom{D_m}{N_{gmi}}$ can be calculated by the gamma function, and $\mathbf{1}(\cdot)$ is the indicator function.
Again, we can calculate the information and use the entropy as a threshold for fraudsters:
\begin{displaymath}
    H\left[d_n=g|\mathcal{M}_{\text{random}})\right] = \sum_{n=1}^N\sum_{m=1}^M\sum_{i=1}^{D_m}\left\{\log \binom{D_m}{N_{gmi}} - N_{gmi}\log D_m\right\}.
\end{displaymath}
Any group with $I(d_n=g|\mathcal{M}_{\text{random}})>(1+\epsilon)\cdot H\left[p(d_n=g|\mathcal{M}_{\text{random}})\right]$ is recognized as fraud, where $\epsilon$ is the tolerance.

\begin{figure*}[!tb]
\centering
    \includegraphics[width=\linewidth]{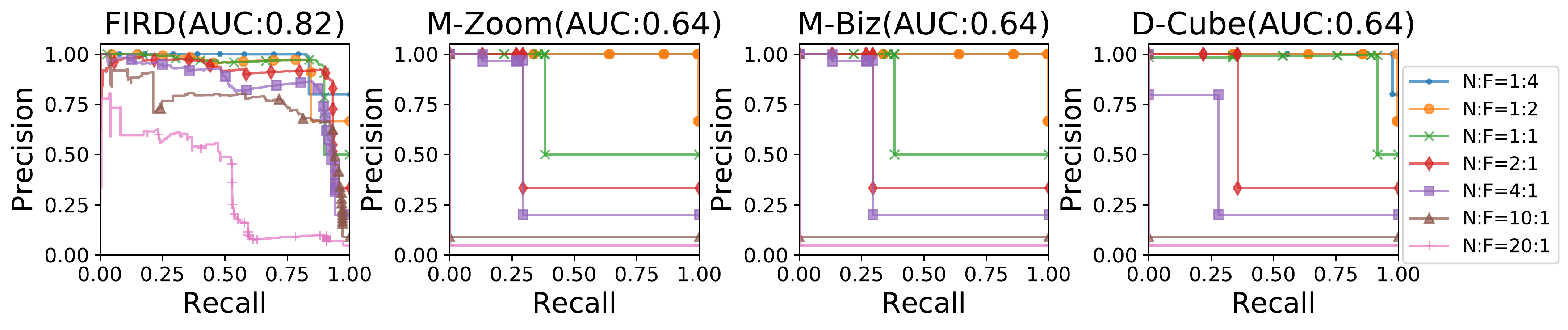}
    \caption{The precision-recall curves on the real-world E-commerce platform data. We present the comparison results with different NFR and plot precision-recall curves. For each method, we evaluate the overall performance by AUC.}
  \label{fig:precision_recall}
\end{figure*}

\subsubsection{Discussion}
\label{sec:discussion}
We first discuss the complexity of FIRD.
FIRD runs in linear time w.r.t. the feature dimension $M$ and the number of data points $N$.
In the E-step, we need to calculate the information of the latent variables, which is $O(NGM)$ for time and space complexity.
In the M-step, the time complexity bottleneck is updating $\boldsymbol{\alpha}$ and $\boldsymbol{\beta}$.
A plain implementation of the update iterations for $\boldsymbol{\alpha}$ and $\boldsymbol{\beta}$ runs in $O(NG\sum_{m}D_m)$.
However, suppose $NM > \sum_mD_m$, we can trade space for time by maintaining an array for $\boldsymbol{\alpha}_{gm}$ and $\boldsymbol{\beta}_{gm}$.
With additional $O(\max_mD_m)$ space, the time cost for updating $\boldsymbol{\alpha}$ and $\boldsymbol{\beta}$ can be reduced to $O(NGM)$.
Therefore, both the time and the space complexities of FIRD are $O(NGM)$.
When the number of samples is extremely large, a single pass of E step and M step is time-consuming.
We can further utilize stochastic EM algorithms to accelerate convergence for large datasets~\cite{nielsen2000stochastic,chen2018stochastic}.

We also provide a parameter initialization strategy.
FIRD adopts pairs of adversarial distributions. 
If one distribution prevails due to the random initialization (for example, $\mu_{gm} > 0.5$ means the sparse distribution prevails, and $\mu_{gm}=1$ is equivalent to removing $Multinomial(\boldsymbol{\beta}_{gm})$), this advantage may persist through the subsequent iterations.
Therefore, a fair start is $\pi_g = 1/G$ and $\mu_{gm}=0.5$ for all possible $g, m$.
The randomness of the model then comes from the random initialization of $\boldsymbol{\alpha}$ and $\boldsymbol{\beta}$.

The hyperparameter $\boldsymbol{\lambda}$ is hard to determine without any upper bound.
Here we discuss how to normalize its value to ease the choice of hyperparameters.
$\boldsymbol{\lambda}$ is the conjugate priors' weight parameter, which controls the sparsity in $\boldsymbol{\alpha}$ and $\boldsymbol{\pi}$ as well as the randomness in $\boldsymbol{\beta}$.
We expect such regularizers to be large enough to punish the models that violate our assumptions but are not so large that they overpower the likelihood.
Since the priors have the same form as the likelihood, the regularizers can be seen as the likelihood of fake data points.
With a given dataset, we then adjust $\boldsymbol\lambda$ so that the fake data points have a comparable size, i.e.
\begin{equation}
\begin{aligned}
    \lambda_g^{(1)} = \lambda^{(1)}\cdot\frac{N}{G},~~
    \lambda_{gmi}^{(2)} = \lambda^{(2)}\cdot\frac{N}{2GD_m}.
\end{aligned}
\end{equation}
Thus, we just need to decide the normalized regularization weights $0<\lambda^{(1)}, \lambda^{(2)}\leq 1$.

\section{Experiments}
\label{sec:experiments}
In this section, we present detailed analysis and two applications of FIRD.
We first demonstrate the application of FIRD on fraud detection.
We compare FIRD to the state-of-the-art unsupervised fraud detection methods on an E-commerce platform dataset and visualize the detection results of FIRD as critical applications like fraud detection require high interpretability.
Then we report the performance of FIRD on anomaly detection benchmark datasets to show its effectiveness as a general anomaly detection method.
Finally, we analyze how the hyperparameters of FIRD affect the performance and its running time cost on synthetic datasets.
%to show its interpretability brought by the adversarial distributions.

We do not compare FIRD with supervised methods for fraud detection as they offer little practical value.
Three primary reasons are: 1) the fraud labels are expensive to collect; 2) the fraud pattern changes as the detection method evolves; and 3) theoretically, the distributions of training data and test data may differ vastly for the two reasons above, which violates the i.i.d. assumption of supervised methods.

Since FIRD models both the synchronization and randomness patterns, it displays superior performance on different tasks.
As such, in our experiments, we compare with the most promising methods on each task.
Specifically, for the fraud detection experiment, we compare with the dense block detection methods M-Zoom, M-Biz, and D-Cube~\cite{shin2016mzoom,shin2017dcube}, which heuristically search for high-density data blocks (with specially designed density definitions).
For anomaly detection methods, we compare FIRD to the state-of-the-art methods such as the histogram-based outlier score (HBOS) \cite{goldstein2012histogram}, the isolation forests (IForest) \cite{liu2008isolation}, the one-class SVM (OCSVM) \cite{dufrenois2016one} and the clustering-based local outlier factor (CBLOF) \cite{he2003discovering}.
For all these experiments, we adopt a {\tt Cython} implementation of FIRD\footnote{A {\tt Cython} implementation is available at https://github.com/fingertap/fird.cython.}.

\subsection{Identify Fraudsters in E-commerce Platform}
\label{sec:shopping}
In this experiment we first describe the experiment setups.
Then we present the comparison result of FIRD with state-of-the-art fraud detection methods on an E-commerce platform dataset.
Finally, we visualize the probabilistic representations learned by FIRD since models are expected to make interpretable decisions in critical applications like fraud detection.

\subsubsection{Experiment Setup}
Every day, a massive number of new users register for E-commerce platforms, a considerable number of which are fraudsters.
We obtained a dataset with over 20,000 fraudsters and a sufficient number of normal user samples with 30 useful features from an E-commerce platform.
For some features, there are tens of thousands of possible values.
The platform labeled the records according to the account behavior in the following few months.
As previously described, the fraudsters form different clusters by sharing different sets of feature values.
In contrast, normal users seldom share feature values and are randomly distributed.

We randomly sample 7 sets of normal users with different sizes and mix them with the fraudsters to synthesize seven datasets with different normal-user-to-fraudster ratios (NFR, ranging from 1:4 to 20:1).
These datasets have the same fraud patterns, and the main difference lies in the noise levels.
We apply FIRD on these datasets and compare its performance to the existing dense block detection methods M-Zoom, M-Biz~\cite{shin2016mzoom}, and D-Cube~\cite{shin2017dcube}.
We use these methods to predict the identities of the data points and evaluate the precision and recall.
We decide $p(\ell|d_n=g)$ for FIRD according to the method described in section~\ref{sec:prior}.
For the dense block detection methods, we assign $p(\ell|d_n=g)=1$ if the proportion of the fraudsters in the detected block exceeds $50\%$, and assign the labels of the detected blocks to their members.
FIRD labels the outliers filtered out according to Section \ref{sec:refine} as normal.

\subsubsection{Overall Comparison Results}
\label{sec:fraud_comparison}

We display the comparison results in \figurename~\ref{fig:precision_recall}.
Note that for the dense block methods, we report the best result under four different density measures.
We plot the precision-recall curve and present the mean area under curve (mean AUC) score for all four methods on these seven datasets.
We observe that, for dense block detection methods, there is a precision decline as the NFR increases.
The AUC score of these methods decreases rapidly as the NFR increases (from $0.99$ to $0.1$).
For comparison, we observe that FIRD is relatively robust to the NFR (AUC ranges from $0.97$ to $0.46$, with an average of $0.82$).
These improvements can be attributed to the refining process, which deletes most of the noisy data points.

As critical applications like fraud detection expect the model to provide interpretable results, we visualize FIRD's detection results with the following concrete example.

\subsubsection{Interpretable Results}
In fraud detection applications, the model is expected to provide interpretable results.
As such, we visualize the probabilistic representations learned by FIRD to show how it captures the fraud patterns.
We first plot the distribution of fraudsters and normal users after detection.
Then we investigate one specific case, i.e., $d_n=1$, to show FIRD captures the fraud pattern.
Finally, we demonstrate the IP address distribution given $d_n=1$ to show that FIRD captures the synchronizations in both fraudsters and some (potentially malicious) normal users.

\begin{figure}[!th]
\centering
\includegraphics[width=\linewidth]{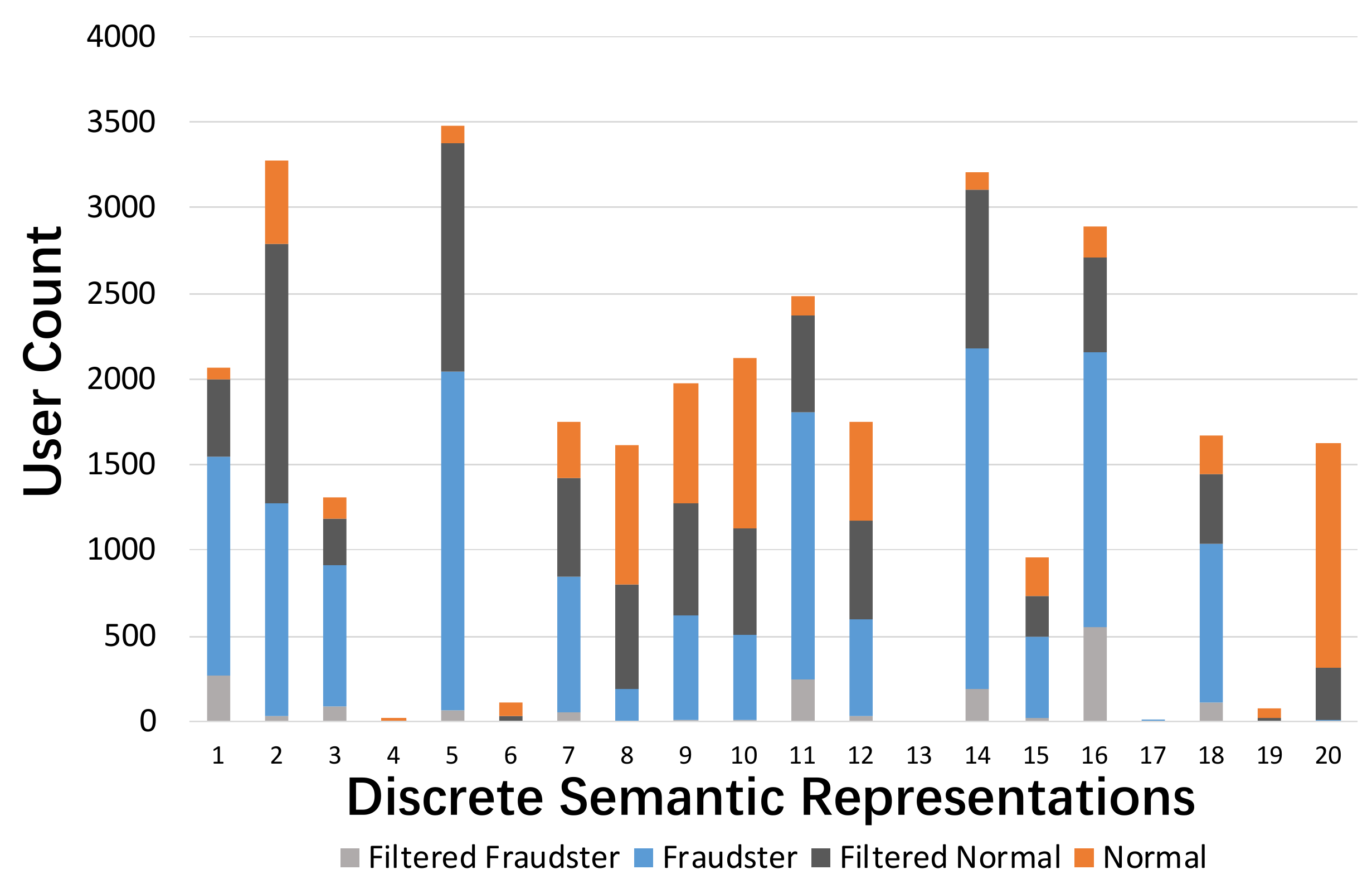}
\caption{X-axis shows 20 bars and each bar is a specific value of the semantic representation (i.e. $d_n=1, \dots, 20$). Y-axis indicates user count and orange, blue and gray blocks respectively illustrate normal users, fraudsters and outliers. Some normal users are not filtered out as \figurename~\ref{fig:res_alpha} described.}
\label{fig:res_mixture_compoenents}
\end{figure}

\figurename~\ref{fig:res_mixture_compoenents} displays the user distribution after detection.
Every data point is assigned according to $\max_g\widetilde{\phi}_{ng}$.
Each bar represents a possible value for $d_n$, and its height indicates the user count.
We color the bars according to the ground truth of each data point, and we discuss the bars from top to bottom.
The top layer (orange blocks) comprises normal users, and the third layer (blue blocks) comprises fraudsters.
We filter out the outliers according to Section \ref{sec:refine} and group them into the second and fourth layers (gray blocks).
We can see that most of the normal users are filtered out of the fraudulent groups (e.g., group 1, 2, and 5).
However, there are still many normal users in these bars. 
As such, for our next step, we investigate one of these bars. 

\begin{figure}[!tb]
\centering
\includegraphics[width=\linewidth]{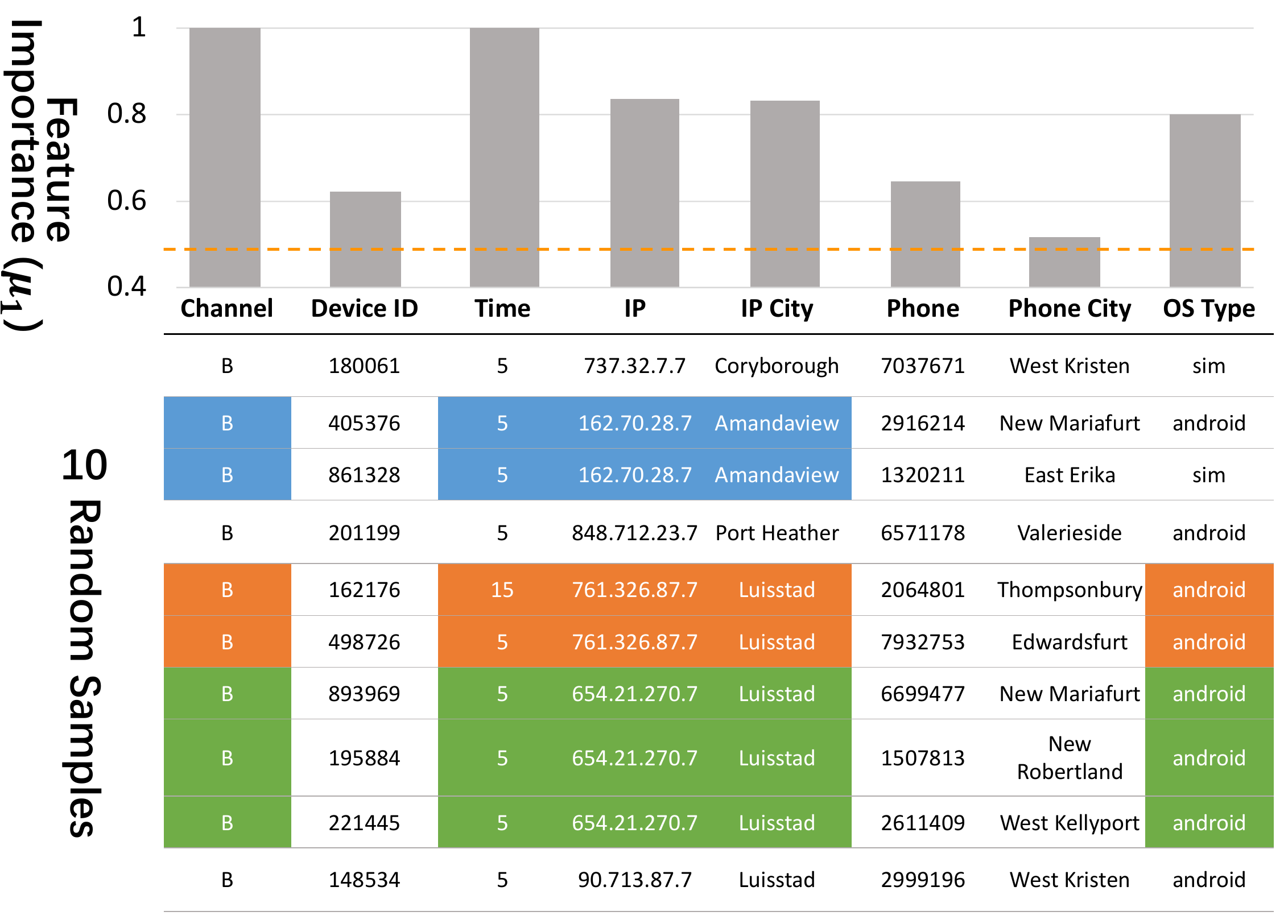}
\caption{Visualization of the representation $d_n=1$ from \figurename~\ref{fig:res_mixture_compoenents}. We randomly sample 10 data points with $d_n=1$ and plot the feature importance (i.e. parameter $\boldsymbol\mu$, as it indicates how synchronized the feature is).
The dashed line $\mu_{1,\cdot}=0.5$ indicates neither the sparsity and randomness patterns overpower the other.
We highlight the synchronized behaviors of the fraudsters in these samples with different background colors.
Note that the three fraud groups (three different background color) share the same synchronized features, which are important features suggested by high $\boldsymbol{\mu}$ values (much greater than 0.5).}
\label{fig:res_records}
\end{figure}

We randomly sample 10 data points from the first bar (i.e. setting $d_n=1$), and plot the learned feature importance parameter ($\boldsymbol\mu_1$) in \figurename~\ref{fig:res_records}.
We can observe the overt synchronized behaviors of these records, which are highlighted as colored blocks.
The parameter $\mu_{1,\cdot}$ for these features is relatively high (over 0.8), which indicates that these features are highly synchronized.
We can see that in the {\it channel} and {\it time} features, the parameter is almost 1, as nearly all of the records with $d_n=1$ have nearly the same value.
The $\mu_{1,IP}$ value for feature {\it IP} is a little smaller than 1 because there are several popular IP addresses.
In features like {\it device ID}, {\it phone} and {\it phone city}, the data points are less synchronized.
As a result, the parameter $\mu_{1, \cdot}$ for these features is closer to 0.5, which signifies that the degrees of synchronization and randomness are similar.

\begin{figure}[!th]
\centering
\includegraphics[width=\linewidth]{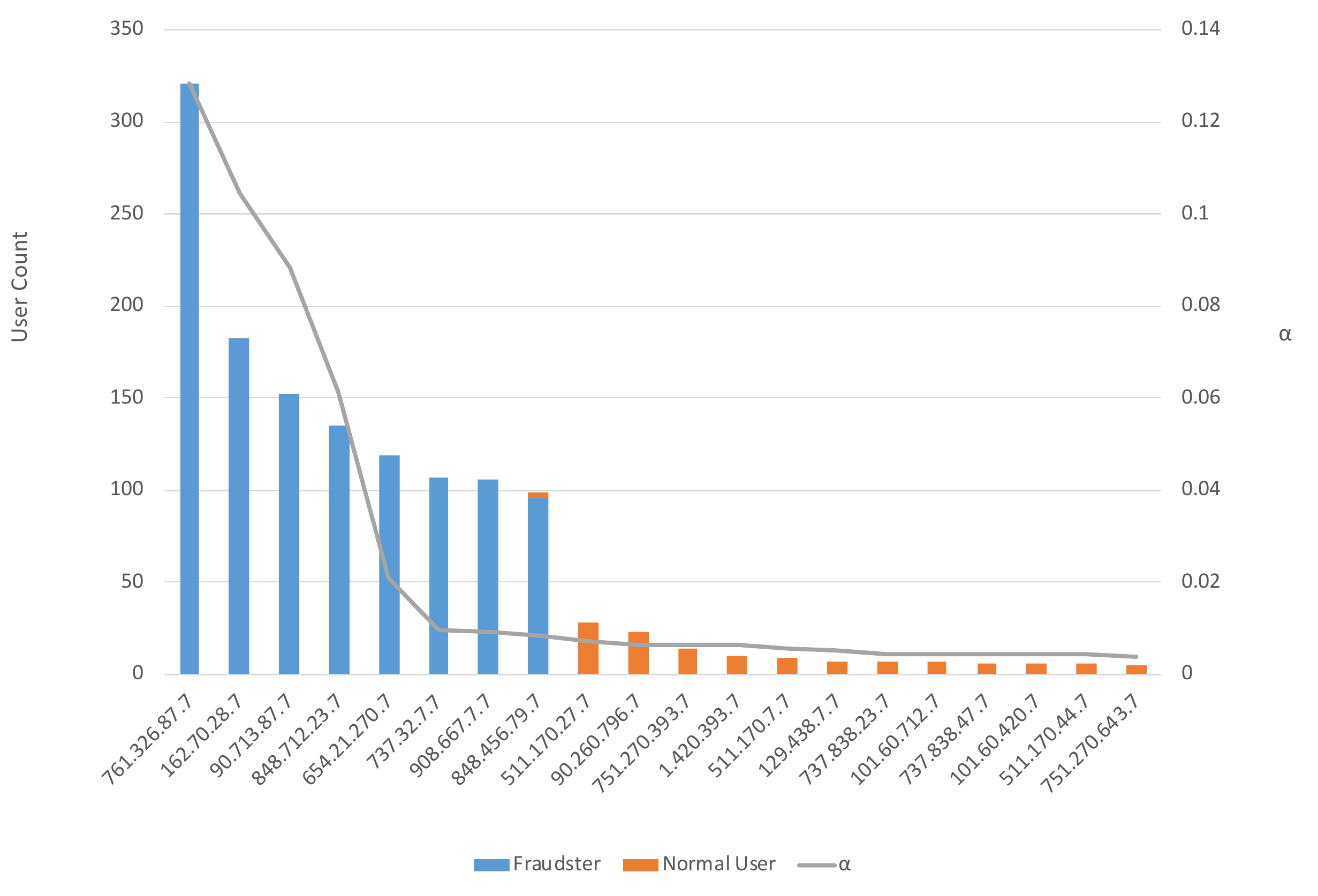}
\caption{Investigate Ip address distribution when setting $d_n=1$. X-axis represents IP address and sorted w.r.t. the parameter $\boldsymbol{\alpha}$. Y-axis is the user count and we color bars according to ground truth and many IP addresses shared by groups of normal users.}
\label{fig:res_alpha}    
\end{figure}

We further investigate the IP feature with $d_n=1$ to see the probability mass of those possible IP addresses ($\alpha_{1, IP,\cdot}$) as well as the number of data points that use these IP addresses.
The result is plotted in \figurename~\ref{fig:res_alpha}.
The bars are again colored according to the ground truth labels of the data points to demonstrate user identities.
We can see that the sampled IPs in \figurename~\ref{fig:res_records} have high probability mass, which means the model correctly captures these clustering patterns.
Interestingly, we find some small groups of normal users that share the same IP addresses.
This explains why some normal users remain in the bars in \figurename~\ref{fig:res_mixture_compoenents}.
The largest group consists of 28 users that share the IP address {\it 511.170.27.7}.
Such synchronization on a single IP address is suspicious, and we should pay extra attention to these groups in practice.
Note that $\boldsymbol{\alpha}_{1,IP}$ decreases faster than the bars since the bars correspond to the likelihood in the objective function, and the regularizers force $\boldsymbol{\alpha}_{1,IP}$ to be sparser than the likelihood.
Together with the observations in \figurename~\ref{fig:res_records}, we verify that FIRD can assign different fraudulent groups the same latent representation as long as the synchronized features of these groups are identical.

In a nutshell, this experiment shows that FIRD can detect reasonable fraud patterns and automatically infer the labels of each data point based on our prior knowledge.
In addition to significant detection results, the probability representations learned by FIRD can also assist us in analyzing the fraud patterns.

\subsection{Results on Anomaly Detection Benchmarks}
\label{sec:compare_outlier}
FIRD can also serve as a general anomaly detection method by assigning $p(\ell|d_n=g)=1$ for anomalies and $p(\ell|d_n=g)=0$ for nominal data points in Eq. \eqref{eq:label}.
Here we demonstrate the detection results of FIRD and comparison methods on benchmark anomaly detection datasets.
We first describe the setup and then provide comparison results and analysis.

\subsubsection{Experiment Setup}
ODDS~\cite{dataset:odds} organizes the benchmark datasets used in related works.
These datasets come from different domains and are readily labeled by reorganizing some multi-class datasets, where the smallest class is chosen to be the anomaly~\cite{abe2006outlier,liu2008isolation,keller2012hics,zimek2013subsampling,aggarwal2015theoretical}.
We use the ROC-AUC score to evaluate the performance of the methods, and compares FIRD with the most promising methods\footnote{More comparison methods: https://pyod.readthedocs.io/en/latest/benchmark.html}, namely the histogram-based outlier score (HBOS) \cite{goldstein2012histogram}, the isolation forests (IForest) \cite{liu2008isolation}, the one-class SVM (OCSVM) \cite{dufrenois2016one} and locally selective
combination in parallel outlier ensembles (LSCP) \cite{zhao2019lscp}.

\subsubsection{Comparison results}

\begin{table}[!tb]
\centering
\begin{tabular}{lccccc}
\toprule
       Dataset &   FIRD &   HBOS &  IForest &  OCSVM   & LSCP\\
\midrule
cardio     &   {\bf 0.949} &  0.843.      &    0.924       &  0.938      &  0.901 \\
musk       &   {\bf 1.000} &  {\bf 1.000} &    0.999       &  {\bf 1.000} &  0.998 \\
optdigits  &   {\bf 1.000} &  0.865       &    0.714       &  0.500       &  -\\
satimage-2 &   {\bf 0.998} &  0.977       &    0.993       &  0.997       &  0.9935\\
shuttle    &   0.990       &  0.986       &    {\bf 0.997} &  0.992       &  0.5514\\
satellite  &   {\bf 0.900} &  0.754       &    0.701       &  0.660       &  0.6015\\
ionosphere  &   {\bf 0.946} &  0.5569       &    0.8529       &  0.8597    &  -\\
pendigits & {\bf 0.972} & 0.9247 & 0.9435 & 0.931 & 0.8744 \\
wbc & 0.944 & {\bf 0.954} & 0.9325 & 0.9376 & 0.945 \\
\bottomrule
\end{tabular}
\caption{Results on anomaly detection benchmark datasets.
We evaluate the performance of identifying fraud users by ROC-AUC scores.
The datasets {\it cardio}, {\it musk} and {\it optdigits} come from~\cite{aggarwal2015theoretical}.
The datasets {\it satellite} and {\it ionosphere} come from~\cite{liu2008isolation}.
The datasets {\it wbc} and {\it pendigits} come from~\cite{keller2012hics}.
The dataset {\it satimage-2} comes from~\cite{zimek2013subsampling}, and {\it shuttle} comes from~\cite{abe2006outlier}.
}
\label{tab:outlier_detection}
\end{table}

We demonstrate the comparison results in Table \ref{tab:outlier_detection}.
FIRD displays competitive performance in most of these benchmarks, which implies that modeling both the synchronization and randomness benefits the detection.
The state-of-the-art anomaly detection methods usually explore the heterogeneous statistical patterns of different features by resampling.
They construct a sequence of datasets by resampling the features as well as data points and learn an independent outlier detector on each resampled dataset~\cite{aggarwal2013outlier}.
They expect that among these samples, the learner may luckily drop the non-informative features and suppress the fraction of outliers so that the detection performance obtained from the ensemble of these detectors is superior to that of learning one detector on the entire dataset.

However, such a sampling strategy consumes more computational resources since it requires training many independent detectors.
The resampling cost to capture the correct feature subset also proliferates with the number of dimensions due to the exponentially many feature combinations.
Besides, the sampling strategy restricts the learner to model the local feature patterns explicitly.
In some features, only a small subset of the dataset display interesting patterns.
If the learner discards these features, we cannot recognize these local patterns.
These data points then become noise, which will further reduce performance.
In comparison, FIRD jointly models the randomness and synchronization to provide better recognition results with a less computational cost.

\subsection{Model Analysis}
\label{sec:hyper_parameter}

In this section, we analyze the effectiveness of FIRD on synchronized datasets, generated following the process in section~\ref{sec:fird}.
We first analyze the dimension $G$ of the semantic variable $\mathbf{d}$.
Then we analyze the weights $\boldsymbol{\lambda}$ of the regularizers.
We finally show the running time of FIRD as the feature dimension $M$ increases.
For parameter analysis, we set $N=20000, G=20, M=20$, and $D_m=200$ for all features.
For running time analysis, we use $N=20000, G=10, D_m=30$ and $M$ ranging from $10$ to $100$.

Since the task is similar to clustering, we apply FIRD, KMeans, and spectral clustering to the synchronized dataset to recover $\mathbf{d}$.
The spectral clustering method~\cite{ng2002spectral} is a strong baseline for clustering performance, and Kmeans~\cite{kanungo2002efficient} is a fast baseline for clustering speed.
We evaluate the clustering performance by three metrics:
\begin{itemize}
  \item {\it Homogeneity} quantifies the pureness of the detected clusters. A high homogeneity score indicates the members of the cluster have almost the same $d_n$.
  \item A high {\it completeness} score means almost all the data points with $d_n=g$ are assigned to the same cluster.
  \item {\it V-score} is the harmonic mean between homogeneity and completeness.
\end{itemize}
All of the three metrics range from 0 to 1, larger values being desirable.
A trivial strategy to achieve the best homogeneity score is to assign each data point as a cluster, while the completeness score will be 0.
We can similarly achieve optimal completeness by assigning all data points to a single cluster, at the cost of zero homogeneity.
Therefore, a good model is expected to optimize the three metrics simultaneously.

\begin{figure}[!tb]
  \centering
  \includegraphics[width=\linewidth]{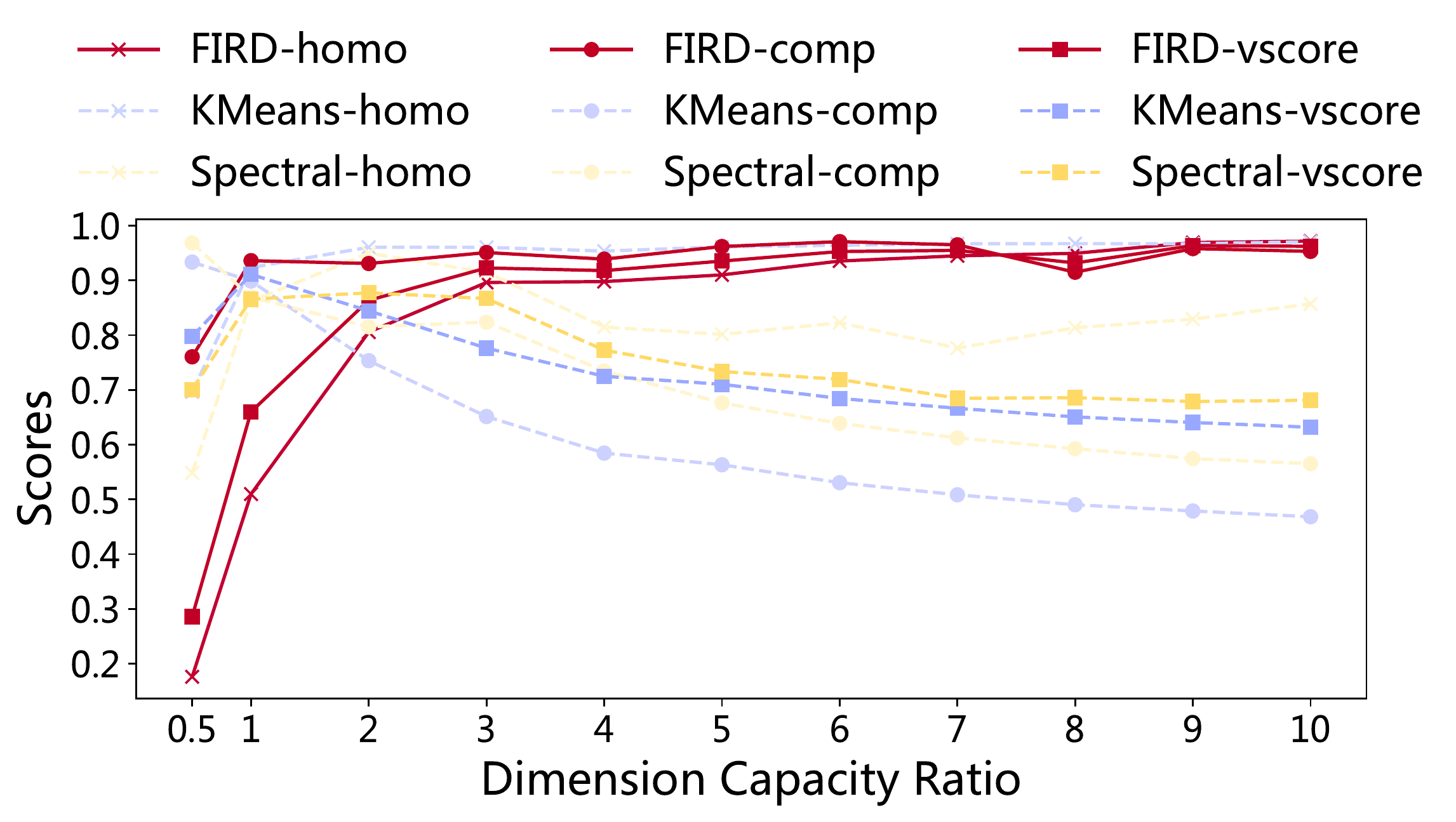}
  \caption{
      Performance under different dimension capacity ratio (DCR).
      DCR is calculated by $G/G_{true}$ for FIRD and $K/G_{true}$ for other two methods, where $G_{true}$ is the ground truth latent dimension $G$ used for data generation.
  }
  \label{fig:diff_G}
\end{figure}

\subsubsection{Different Choice of $G$}

We first study the different choices of the dimension $G$ of the latent discrete space.
The parameter $G$ controls the dimension of latent discrete space of $\mathbf{d}$, which is similar to the $K$ of the KMeans and spectral clustering algorithm.
For normalization, we demonstrate the result with different {\it dimension capacity ratio} (DCR) defined as the ratio of the dimension $G$ to the ground truth dimension $G_{true}$, so the results are independent to $G_{true}$.
We synchronize the dataset according to the generation process described in section~\ref{sec:fird} with $G_{true}=10$.
We then adopt $G=5, 10, 20, \dots, 100$ for FIRD, and the same for $K$ in KMeans and spectral clustering method to obtain DCR ranging from $0.5$ to $10$.
Note that for KMeans and spectral clustering, we calculate the distance in the discrete space by applying a one-hot encoding preprocessing step to the vectors.
The detection performance is demonstrated in \figurename~\ref{fig:diff_G}.
It indicates that when the patterns are locally different for the clusters, conventional methods such as KMeans and spectral clustering require a perfect guess of $G_{true}$, i.e., DCR = 1, to achieve relatively high performance.
As DCR increases, conventional methods tend to split the clusters into smaller ones due to the random noise in non-informative dimensions.
In contrast, FIRD achieves high scores under all three metrics as long as the dimension capacity ratio is large enough, e.g., $G/G_{true} > 2$.
As described in Section~\ref{sec:fird}, FIRD automatically determines an appropriate $G$ through the sparsity of $\boldsymbol{\pi}$, so increasing the DCR does not affect its effectiveness.

\begin{figure}[!t]
\centering
  \includegraphics[width=\linewidth]{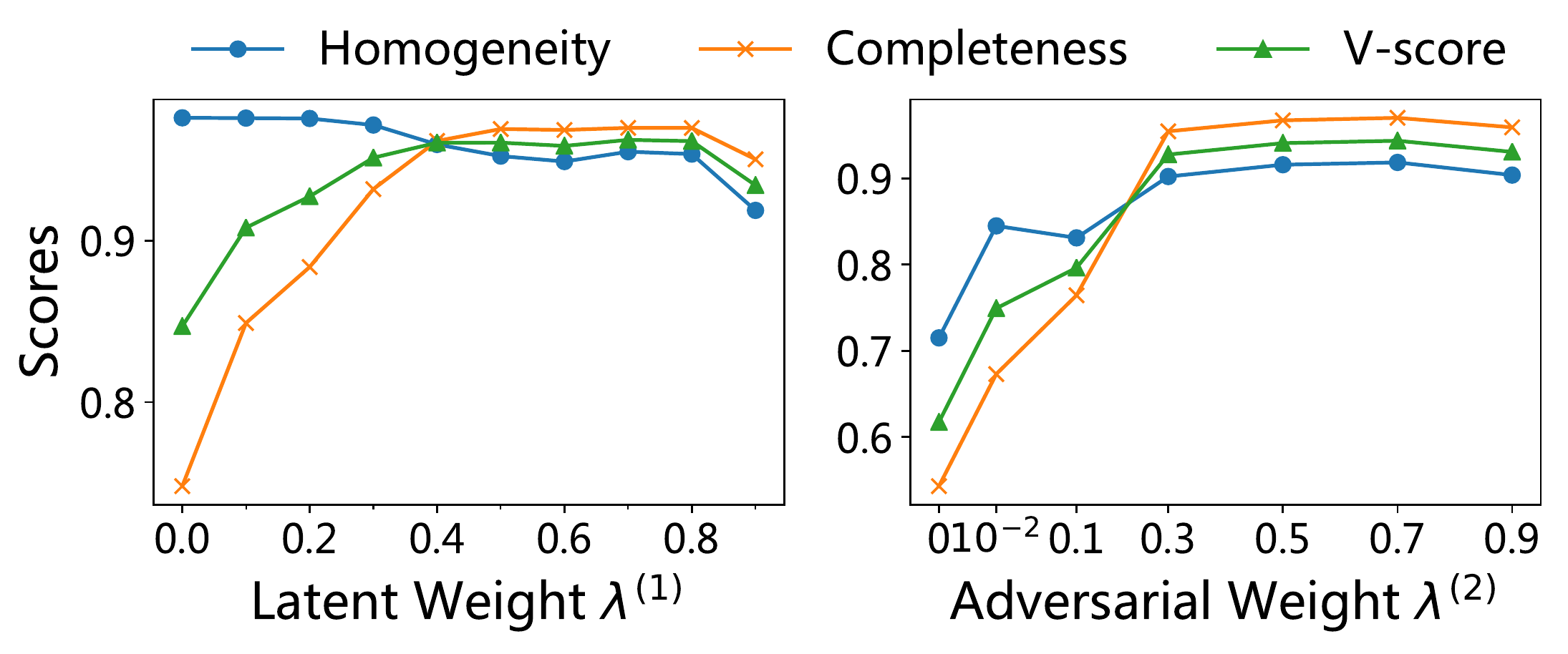}
  \caption{
      Performance under different regularization weights.
      FIRD is robust to the choice of the weights as long as the weights are not too small to cancel the effect of regularization.
      }
  \label{fig:different_lamb}
\end{figure}

\subsubsection{Different Choice of $\lambda$}

We introduced the normalized regularization weights $\boldsymbol{\lambda}^{(1)}$ and $\boldsymbol{\lambda}^{(2)}$ in section~\ref{sec:fird} and~\ref{sec:discussion} to control the degree of sparsity or randomness in the parameters.
The sparsity in $\boldsymbol{\pi}$ controlled by $\boldsymbol{\lambda}^{(1)}$ enables the automatic determination of the appropriate $G$, and $\boldsymbol{\lambda}^{(2)}$ reflects the difference between the adversarial distributions.
We apply FIRD with different $\boldsymbol{\lambda}$s to the synchronized dataset to study the effect, with results shown in~\figurename~\ref{fig:different_lamb}.
We find that FIRD is robust to the choice of both $\boldsymbol{\lambda}^{(1)}$ and $\boldsymbol{\lambda}^{(2)}$ except when the weights are too small so that the learner no longer enjoys the modeling ability of the adversarial distributions.

\begin{figure}[!tb]
\centering
  \includegraphics[width=\linewidth]{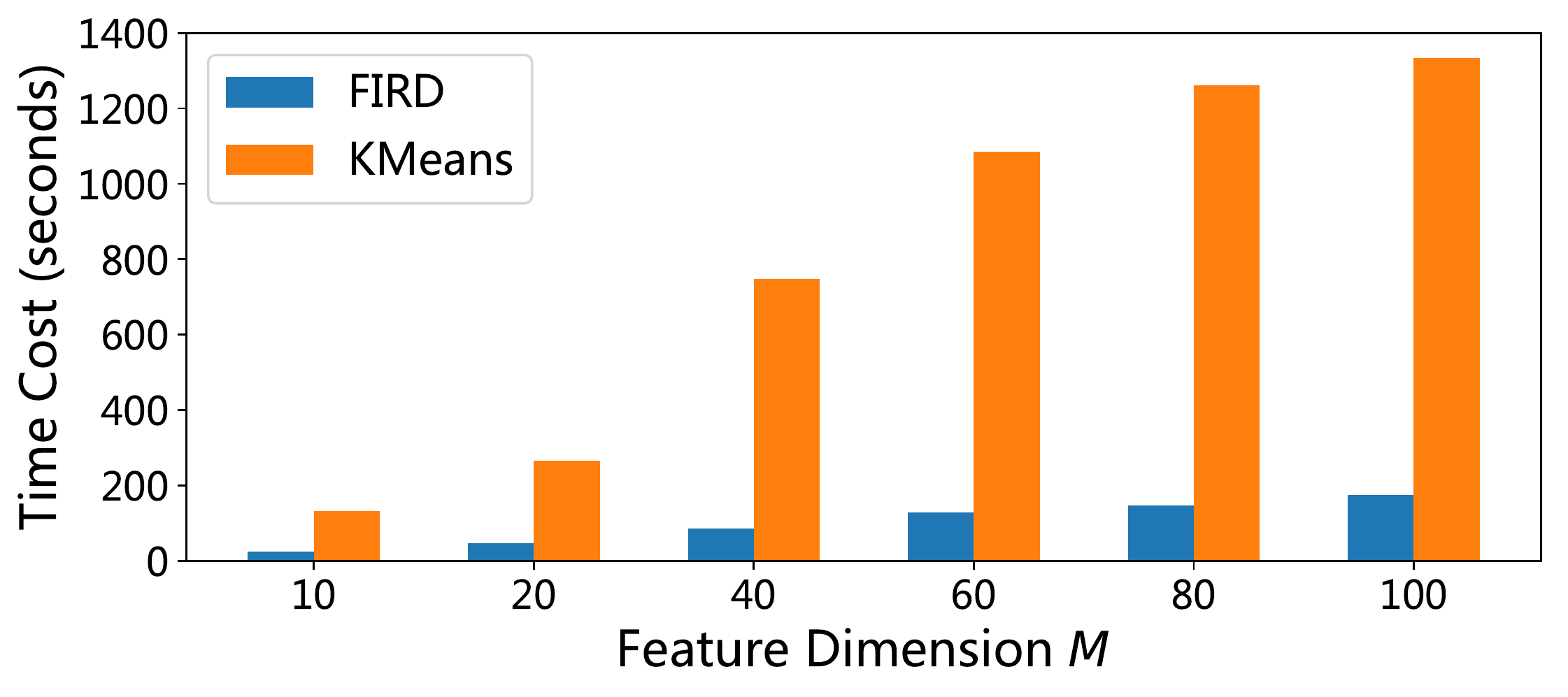}
  \caption{
  Running time of FIRD and KMeans as the feature dimension $M$ increases.
  We implement FIRD with {\tt Cython} and adopt KMeans implementation in the well-known Python package  {\tt scikit-learn}~\cite{scikit-learn}.
  }
  \label{fig:running_time}
\end{figure}

\subsubsection{Running Time Analysis}

FIRD has linear computational cost w.r.t. the dataset scale, especially in high-dimensional spaces.
We demonstrate the running time comparison of FIRD and KMeans in \figurename~\ref{fig:running_time}.
Here we adopt KMeans implementation in the well-known Python package  \texttt{scikit-learn} \cite{scikit-learn}.
Since KMeans calculates the distance between the one-hot encoded vectors, the computational cost increases exponentially with the feature dimension $M$.
In contrast, FIRD decouples the features by the independence assumption, so it enjoys a linear growth in running time as $M$ increases.

%In conclusion, the experiments on synthetic datasets suggest that the performance of FIRD is robust to the choice of $G$ and $\lambda$, and the low time complexity makes FIRD scalable to large datasets. 

\section{Conclusion}
\label{sec:conclusion}
In this paper, we propose a novel unsupervised generative learning framework called FIRD to model heterogeneous statistical patterns in unlabeled datasets.
FIRD utilizes the adversarial distributions with priors to capture such patterns.
In discrete spaces, FIRD captures the synchronization and randomness patterns, which turns out quite useful on both fraud detection and general anomaly detection applications.
The significant results on various datasets verify that modeling heterogeneous statistical patterns provides more generalizable representations and benefits various downstream applications.
As future work, we expect FIRD to be effective in other applications that model the patters other than the synchronization and randomness by adopting appropriate adversarial distributions.

\section{Acknowledgments}
This research was supported by Alibaba Group through Alibaba Innovative Research Program.
\appendix

\section{Proof of Theorem~\ref{thm:unique_optimizer}}
\label{sec:theorem_proof}
As the ELBO is concave w.r.t. all variables, it can be easily shown that the solution in Theorem~\ref{thm:unique_optimizer} is optimal.
To show the uniqueness, we prove that each adversarial component $p_k(x_{nm}|d_n, \theta)$ will fit the corresponding pattern $p_{k'}^\star(x_{nm}|d_n)$.
According to the definition 
\begin{displaymath}
    D_{KL}\left(p^\star_{k'}(x_{nm}|d_n)\middle\Vert p_k(x_{nm}|d_n, \theta)\right) = \mathbb{E}_{x\sim p^\star_{k'}}[\log p_k(x_{nm}|d_n, \theta)] + H,
\end{displaymath}
where $H$ is the entropy of $p^\star_{k'}(x_{nm}|d_n)$, the KL divergence assumption in Theorem~\ref{thm:unique_optimizer} indicates for all possible $d_n$ and $\forall j\neq k$,
\begin{equation}
\label{eq:likelihood_inequality}
\begin{aligned}
\mathbb{E}_{x\sim p^\star_{k'}}[p_k(x_{nm}|d_n, \theta)] > \mathbb{E}_{x\sim p^\star_{k'}}[p_j(x_{nm}|d_n, \theta)].
\end{aligned}
\end{equation}
Since $\tilde{\mu}_{nmk} > 0$, multiplying Eq.~\eqref{eq:likelihood_inequality} with $q(d_n)\tilde{\mu}_{nmk}$ and summing over $m, k$ and $d_n$ gives
\begin{displaymath}
\begin{aligned}
&\lim_{N\rightarrow\infty}\frac{1}{N}ELBO(p_{k}\rightarrow p_{k'}^\star) \\
=&\lim_{N\rightarrow\infty}\frac{1}{N}\sum_{n, d_n, m, k}q(d_n)\tilde{\mu}_{nmk}\log p_k(x_{nm}|d_n, \theta) + const\\
\geq&\lim_{N\rightarrow\infty}\frac{1}{N}\sum_{n, d_n, m, k}q(d_n)\tilde{\mu}_{nmk}\log p_j(x_{nm}|d_n, \theta) + const\\
=&\lim_{N\rightarrow\infty}\frac{1}{N}ELBO(p_{j}\rightarrow p_{k'}^\star)
\end{aligned}
\end{displaymath}
which indicates that the optimal solution use $p_k(x_{nm}|d_n, \theta)$ to approximate the corresponding pattern $p_{k'}^\star(x_{nm}|d_n)$.
Using the EM algorithm gives the estimation of other parameters in Theorem~\ref{thm:unique_optimizer}, which completes the proof.

\bibliographystyle{ACM-Reference-Format}
\bibliography{dblp}
\end{document}